\documentclass[conference]{IEEEtran}
\IEEEoverridecommandlockouts
\usepackage{cite}
\usepackage{amsmath,amssymb,amsfonts}
\usepackage{algorithmic}
\usepackage{algorithm}
\usepackage{graphicx}
\usepackage{textcomp}
\usepackage{xcolor}
\usepackage{bm}
\usepackage{hyperref}
\usepackage{url}
\usepackage{booktabs}

\newtheorem{theorem}{Theorem}

\newtheorem{lemma}[theorem]{Lemma}

\newtheorem{definition}[theorem]{Definition}

\newtheorem{remark}[theorem]{Remark}

\newenvironment{proof}{\noindent\textit{Proof:}}{\hfill$\square$\medskip}

\def\BibTeX{{\rm B\kern-.05em{\sc i\kern-.025em b}\kern-.08em
    T\kern-.1667em\lower.7ex\hbox{E}\kern-.125emX}}

\begin{document}

\title{Decentralized Multi-Player Q-Learning in Episodic Markov Decision Processes with Information Asymmetry}

\author{
\IEEEauthorblockN{1\textsuperscript{st} Larissa Xu\textsuperscript{*}}
\IEEEauthorblockA{\textit{Department of Mathematics} \\
\textit{University of California, Los Angeles}\\
Los Angeles, USA \\
xuzhiyun004119@g.ucla.edu}
\and
\IEEEauthorblockN{2\textsuperscript{nd} King Bi}
\IEEEauthorblockA{\textit{Department of Computer Science} \\
\textit{University of California, Los Angeles}\\
Los Angeles, USA \\
king0508@g.ucla.edu}
\and
\IEEEauthorblockN{3\textsuperscript{rd} William Chang}
\IEEEauthorblockA{\textit{Department of Mathematics} \\
\textit{University of California, Los Angeles}\\
Los Angeles, USA \\
chang314@g.ucla.edu}
}


\maketitle

\begin{abstract}
We study decentralized multi-player reinforcement learning in episodic tabular Markov decision processes (MDPs) under three forms of information asymmetry: (A) unobserved actions with common rewards, (B) observed actions with independent rewards, and (C) unobserved actions with independent rewards. Players cannot communicate during learning but may agree on a protocol a priori. For Problems A and B we propose \texttt{mQ-learning} and \texttt{mQ-learning-intervals}, achieving $\tilde{O}(\sqrt{H^4 S A_{\text{joint}}\, T})$ regret, where $H$ is the horizon, $S$ the state count, $T = KH$ the total steps, and $A_{\text{joint}} = \prod_{i=1}^M |\mathcal{A}_i|$ the joint action space across $M$ players. For Problem C we give \texttt{mEXC} and \texttt{mEXC-Bellman}, two-phase explore-then-commit algorithms with regret $\tilde{O}(H (S A_{\text{joint}})^{1/3} T^{2/3})$. Against the centralized joint-action benchmark, decentralized learning under information asymmetry matches the single-agent Q-learning rate of \cite{jin2018q} up to logarithmic factors. Because $A_{\text{joint}}$ grows exponentially in $M$, the bounds are most meaningful for small $M$ or small per-player action sets.
\end{abstract}

\begin{IEEEkeywords}
multi-player reinforcement learning, Markov decision processes, Q-learning, information asymmetry, regret bounds, decentralized learning
\end{IEEEkeywords}

\section{Introduction}

Multi-agent reinforcement learning (MARL) arises in cooperative systems such as communication networks \cite{anandkumar2011distributed}, multi-robot coordination, and distributed resource allocation \cite{sutton2018reinforcement}, where agents must coordinate without centralized control or explicit communication during learning.

The multi-armed bandit (MAB) framework is well understood in the single-player setting \cite{lai1985asymptotically, auer2002finite}. Extensions to the multi-player setting \cite{bistritz2018distributed} reveal that information asymmetry between players introduces substantial coordination challenges. Prior work on cooperative multi-player bandits has studied several asymmetry models \cite{chang2021online, chang2023optimal}, but these results are limited to the stateless bandit setting, whereas many applications have state structure that agents must learn to navigate.

In the single-agent episodic MDP setting, Jin et al.\ \cite{jin2018q} showed that a model-free Q-learning algorithm with upper confidence bounds achieves $\tilde{O}(\sqrt{H^4SAT})$ regret, near-optimal up to polynomial factors in $H$. A natural question is whether similar guarantees hold for the multi-player setting with information asymmetry, or whether asymmetry imposes an additional cost beyond the centralized joint-action rate.

The key insight, first observed in the bandit setting \cite{chang2021online}, is that players can use a \emph{pre-agreed deterministic protocol} to implicitly coordinate without communication: if all players follow the same deterministic algorithm with access to the same information, they independently reach the same decisions. We extend this to episodic MDPs, where coordination must also account for state transitions and multi-step planning.

\paragraph{Our Contributions}
\begin{itemize}
    \item For Problem A (unobserved actions, common rewards), \texttt{mQ-learning} uses a lexicographic ordering of joint actions to coordinate implicitly and achieves $\tilde{O}(\sqrt{H^4 S A_{\text{joint}}\, T})$ regret (Theorem~\ref{thm:asymmetry_action}), where $A_{\text{joint}} = \prod_{i=1}^M |\mathcal{A}_i|$.

    \item For Problem B (observed actions, independent rewards), \texttt{mQ-learning-intervals} maintains upper and lower confidence bounds at each state-action pair to coordinate action elimination across players, with the same $\tilde{O}(\sqrt{H^4 S A_{\text{joint}}\, T})$ regret (Theorem~\ref{thm:asymmetry_reward_main}).

    \item For Problem C (unobserved actions, independent rewards), \texttt{mEXC} and \texttt{mEXC-Bellman} are explore-then-commit algorithms. With $K' \asymp (S A_{\text{joint}})^{1/3} K^{2/3}$, we obtain $\tilde{O}(H (S A_{\text{joint}})^{1/3} T^{2/3})$ regret (Theorem~\ref{thm:asymmetry_full}).

    \item We prove auxiliary lemmas on weighted learning rates and interval widths that may be of independent interest, and show that Algorithm~\ref{alg:mq} is operationally equivalent to centralized joint-action Q-learning—the agreement is achieved \emph{implicitly} through deterministic tie-breaking rather than through observation.
\end{itemize}

\paragraph{Role of $M$.} Our bounds match the single-agent rate \emph{relative to the joint-action benchmark}, but grow exponentially in $M$ through $A_{\text{joint}}$. This is unavoidable in tabular MARL: even a centralized learner has $\Omega(\sqrt{S A_{\text{joint}} T})$ regret \cite{jin2018q}. The contribution is therefore that asymmetry imposes no \emph{additional} factor beyond the centralized rate, not that asymmetry is free in absolute terms.

\paragraph{Organization.} Section~\ref{sec:preliminary} sets up the model. Section~\ref{sec:results} treats Problems A and B; Section~\ref{sec:problem_c} treats Problem C. Section~\ref{sec:related} discusses related work; proofs are in the Appendix.

\section{Preliminaries}\label{sec:preliminary}

\subsection{Multi-Player Episodic MDP Model}

We consider a multiplayer episodic tabular MDP $(\mathcal{P}, \mathcal{S}, \mathcal{A}, H, \mathbb{P}, r)$. Here $\mathcal{P} = \{P_1, \ldots, P_M\}$ is the set of $M$ players and $\mathcal{S}$ is the state space with $|\mathcal{S}| = S$. Each player $P_i$ has actions $\mathcal{A}_i$, and the joint action space is $\mathcal{A} = \mathcal{A}_1 \times \cdots \times \mathcal{A}_M$ with $|\mathcal{A}| = A_{\text{joint}} = \prod_{i=1}^M |\mathcal{A}_i|$, which scales as $A_{\max}^M$ in the worst case—exponential in $M$. We use $A$ and $A_{\text{joint}}$ interchangeably in regret bounds. $\mathbb{P}_h$ is the set of unknown transition kernels at step $h$, one for each joint action $\bm{a} \in \mathcal{A}$; entry $(\mathbb{P}_h)_{ij}(\bm{a})$ is the probability of moving from $x_i$ to $x_j$ under $\bm{a}$. The horizon is $H$, and the reward $r$ depends on the joint action. Without loss of generality, states are layered: the initial state is $s_0$ and the final state is $s_H$.

At each time step, each player observes the current state and picks an arm from their set. The $M$-tuple of arms is denoted $\bm{a} = (a_1, \ldots, a_M)$. Throughout this paper, vectors and tuples are denoted in boldface. This generates a random reward $X_{\bm{a}} \in [0,1]$ from a $1$-subgaussian distribution $F_{\bm{a}}$ with mean $\mu_{\bm{a}}$. After performing the joint action, the players transition to a new state according to the unknown transition kernel $\mathbb{P}_h$.

\subsection{Value Functions and Bellman Equation}

A \emph{policy} $\pi$ specifies, for each step $h$ and state $x$, a joint action $\bm{a}$. The \emph{value function} of a policy $\pi$ at step $h$ and state $x$ is defined as
\begin{equation}\label{eq:value}
    V_h^\pi(x) = \mathbb{E}_\pi\left[\sum_{h'=h}^H r_{h'}(x_{h'}, \bm{a}_{h'}) \;\middle|\; x_h = x\right]
\end{equation}
and the corresponding \emph{action-value function} (Q-function) is
\begin{equation}\label{eq:qvalue}
    Q_h^\pi(x, \bm{a}) = r_h(x, \bm{a}) + \mathbb{E}_{x' \sim \mathbb{P}_h(\cdot | x, \bm{a})}[V_{h+1}^\pi(x')].
\end{equation}

The optimal value function satisfies the \emph{Bellman optimality equation}:
\begin{equation}\label{eq:bellman}
    Q_h^*(x, \bm{a}) = r_h(x, \bm{a}) + [\mathbb{P}_h V_{h+1}^*](x, \bm{a}),
\end{equation}
where $V_h^*(x) = \max_{\bm{a}} Q_h^*(x, \bm{a})$ and $[\mathbb{P}_h f](x, \bm{a}) = \mathbb{E}_{x' \sim \mathbb{P}_h(\cdot|x,\bm{a})}[f(x')]$.

\subsection{Regret and Learning Objective}

The players want to collectively identify the best policy, i.e., one that returns the highest rewards in expectation. This takes into account both the true reward means and the probability transition matrices. However, the players know neither the means $\mu_{\bm{a}}$ nor the transition functions $\mathbb{P}_h$. They must learn by playing and exploring. We capture learning efficiency via the per-player \textit{expected regret}:
\begin{equation}\label{eq:regret}
    R_T = \sum_{k=1}^K \left[V_1^*(x_1^k) - V_1^{\pi_k}(x_1^k)\right]
\end{equation}
where $K$ is the number of episodes, $T = KH$ is the total number of steps, and $\pi_k$ is the policy used in episode $k$. Our goal is to design decentralized algorithms with sublinear regret. Fundamental results for single-player MAB problems \cite{lai1985asymptotically} suggest an $\Omega(\log T)$ lower bound for the multi-player problem as well. In the episodic MDP setting, $\tilde{\Omega}(\sqrt{H^2SAT})$ is a known lower bound \cite{jin2018q}. We specifically exclude any explicit communication between players during learning, though they may coordinate \textit{a priori}.

\subsection{Information Asymmetry Models}

We consider three types of information asymmetry, following \cite{chang2021online}.

\paragraph{Problem A: Action Asymmetry with Common Rewards}\label{problem_A}
Consider a set of $M$ players $P_1, \ldots, P_M$, where player $P_i$ has a set $\mathcal{K}_i$ of $K_i$ arms. At each time instant, each player picks an arm from their set with the $M$-tuple denoted by $\bm{a} = (a_1, \ldots, a_M)$. The reward to \emph{all} players equals $X_{\bm{a}(t)}(t)$, i.e., the reward is common, depends on the joint action $\bm{a}(t)$, and is independent across time. Crucially, no player can observe the actions of the other players. Formally, the information available to player $i$ at time $t$ is: $\{x_h^k, r_h^k : k \leq t, h \in [H]\}$ together with their own action history $\{a_i^{h,k} : k \leq t, h \in [H]\}$.

\paragraph{Problem B: Reward Asymmetry with Observed Actions}
Each agent can observe the actions of all other agents, but the rewards are independent. If the arms pulled at time $t$ are $\bm{a}(t)$, each player gets an i.i.d.\ copy of the reward: player $i$ receives $X^i_{\bm{a}(t)}(t)$ drawn independently from $F_{\bm{a}(t)}$. Since rewards are i.i.d., the expected regret is the same for all players. The information available to player $i$ is: $\{x_h^k, \bm{a}_h^k, r_h^{i,k} : k \leq t, h \in [H]\}$, i.e., full action information but only their own rewards.

\paragraph{Problem C: Full Asymmetry}
No player can observe the actions of the others, and rewards are independent from the same distribution. This combines the challenges of both Problems A and B. The information available to player $i$ is only: $\{x_h^k, a_i^{h,k}, r_h^{i,k} : k \leq t, h \in [H]\}$.

\begin{remark}
Problem C is the most general of the three, and Problem A is not a special case of Problem B (or vice versa), since they involve different types of asymmetry. The hierarchy is: Problem C generalizes both A and B.
\end{remark}

\section{Main Results}\label{sec:results}

\subsection{Problem A: Asymmetry in Actions}

The key difficulty in Problem A is that players cannot observe each other's actions and must therefore coordinate implicitly. We use a lexicographic ordering from \cite{chang2021online} to break ties deterministically.

\begin{definition}[Lexicographic Order]\label{def:order}
For two $M$-tuples $\bm{x} = (x_1, \ldots, x_M)$ and $\bm{y} = (y_1, \ldots, y_M)$, we say $\bm{x} < \bm{y}$ if and only if there exists an $n$ such that $\forall i < n$, $x_i = y_i$, and $x_n < y_n$.
\end{definition}

Intuitively, $\bm{x} < \bm{y}$ if $\bm{x}$ is smaller than $\bm{y}$ when viewed as an $M$-digit number. When each player has the same number of actions $A$, we can view $\bm{a} \in \mathcal{A}$ as a number in base $A$. This definition remains applicable even when players have different numbers of actions.

The key observation is that in Problem A, all players receive the \emph{same reward} and observe the \emph{same state transitions}. Therefore, if all players run the same deterministic algorithm with the same tie-breaking rule, they will independently maintain identical Q-value estimates and select the same joint action at every step---without needing to observe each other's actions.

\begin{algorithm}[t]
\caption{\texttt{mQ-learning} (for each player $P_i$)}\label{alg:mq}
\begin{algorithmic}[1]
\STATE Initialize $Q^k_h(i, x, \bm{a}) \leftarrow H$, $N_h(i, x, \bm{a}) \leftarrow 0$ for all $(x, \bm{a}, h) \in \mathcal{S} \times \mathcal{A} \times [H]$.
\FOR{episode $k = 1, \ldots, K$}
\STATE Receive initial state $x_1$.
\FOR{step $h = 1, \ldots, H$}
\STATE Pick $\bm{a}_h$ as the \emph{smallest} action (via Definition~\ref{def:order}) in $\arg\max_{\bm{a}'} Q^{k}_h(i, x_h, \bm{a}')$.
\STATE Observe common reward $r_h$ and next state $x_{h+1}$.
\STATE $t = N_h(i, x_h, \bm{a}_h) \leftarrow N_h(i, x_h, \bm{a}_h) + 1$.
\STATE $b_t \leftarrow c\sqrt{H^3 \iota / t}$.
\STATE $Q_h^k(i, x_h, \bm{a}_h) \leftarrow (1 - \alpha_t)Q_h^{k-1}(i, x_h, \bm{a}_h)$ \\ $\quad + \alpha_t[r_h + V_{h+1}^{k-1}(x_{h+1}) + b_t]$.
\STATE $V_h^k(x_h) \leftarrow \min\{H, \max_{\bm{a}'} Q^k_h(i, x_h, \bm{a}')\}$.
\ENDFOR
\ENDFOR
\end{algorithmic}
\end{algorithm}

The learning rate is $\alpha_t = (H+1)/(H+t)$ and $\iota = \log(SA_{\text{joint}}T/p)$. The algorithm is essentially the UCB-based Q-learning of \cite{jin2018q} applied to the joint action space, with the deterministic tie-breaking of Definition~\ref{def:order} ensuring implicit coordination. Operationally, Algorithm~\ref{alg:mq} matches centralized joint-action Q-learning: no player observes the others' actions, but because rewards and transitions are common and tie-breaking is deterministic, every player's Q-table evolves identically and they select the same joint action at every step.

\paragraph{Synchronization under repeated tie-breaking.}
Stability of implicit coordination follows from a one-line invariant. All players initialize $Q^1_h \equiv H$, so the $\arg\max$ set is identical and the lexicographic-smallest element is uniquely determined. If $Q^k_h$ is identical across players at the start of episode $k$, then in Problem A the update on line~8 depends only on quantities common to all players (same reward $r_h$, same next state $x_{h+1}$, same visit count $t$), so $Q^{k+1}_h$ remains identical. The invariant is preserved across all $KH$ steps; no probabilistic argument is required.

\begin{theorem}\label{thm:asymmetry_action}
Consider Problem A with $M$ players and joint action space size $A_{\text{joint}} = \prod_{i=1}^M |\mathcal{A}_i|$. There exists $c > 0$ such that for any $p \in (0,1)$, with $b_t = c\sqrt{H^3\iota/t}$ and $\alpha_t = (H+1)/(H+t)$, with probability at least $1 - p$ the total regret of \texttt{mQ-learning} for any player is $O(\sqrt{H^4 S A_{\text{joint}}\, T\, \iota})$, where $\iota = \log(S A_{\text{joint}} T / p)$.
\end{theorem}

\begin{remark}
The bound matches the single-agent rate of \cite{jin2018q} against the joint-action benchmark: common rewards plus deterministic tie-breaking are enough to recover the centralized rate without communication. We do not claim asymmetry is free in an absolute sense, since $A_{\text{joint}}$ itself grows exponentially in $M$. In a distributed system (e.g.\ $M$ transmitters selecting channels, or $M$ robots picking sub-tasks), the per-episode suboptimality decays as $\tilde{O}(H^2\sqrt{S A_{\text{joint}} / K})$, the same scaling a single centralized controller would achieve.
\end{remark}

\subsection{Problem B: Asymmetry in Rewards}

When players receive independent rewards, they cannot maintain identical Q-value estimates. Even though they observe the same actions, their different reward realizations lead to different confidence intervals for each state-action pair. We address this by maintaining both upper and lower confidence bounds, inspired by the approach of \cite{chang2023optimal} in the bandit setting.

The key idea is that players maintain a ``desired set'' of plausible optimal joint actions at each state. An action is eliminated from the desired set when some player's confidence interval shows it is definitively suboptimal. Since players can observe each other's actions, a unilateral deviation from the agreed-upon action serves as an \emph{implicit communication signal}: it tells all players that the deviating player has determined the proposed action is suboptimal.

\begin{algorithm}[t]
\caption{\texttt{mQ-learning-intervals} (for each player $P_i$)}
\begin{algorithmic}[1]
\STATE Initialize $Q^{k,\text{up}}_h(i, x, \bm{a}) \leftarrow H$, $Q^{k,\text{low}}_h(i, x, \bm{a}) \leftarrow 0$, $N_h(i, x, \bm{a}) \leftarrow 0$ for all $(x, \bm{a}, h) \in \mathcal{S} \times \mathcal{A} \times [H]$.
\FOR{episode $k = 1, \ldots, K$}
\STATE Receive initial state $x_1$.
\FOR{step $h = 1, \ldots, H$}
\STATE In state $x_h$, consider the action in the desired set pulled the fewest times. Break ties via Definition~\ref{def:order}. Call this $\bm{a}$.
\STATE If $\exists\, \bm{a}'$ such that for some player $i$: $Q_h^{k,\text{up}}(i, \bm{a}, x_h^k) < Q_h^{k,\text{low}}(i, \bm{a}', x_h^k)$, then player $i$ intentionally deviates from $\bm{a}$.
\STATE Call the realized joint action $\bm{a}_h$. All players observe $\bm{a}_h$.
\STATE If $\bm{a}_h \neq \bm{a}$, all players eliminate $\bm{a}$ from the desired set for state $x_h$.
\STATE Each player observes own reward $r_h^i$ and state $x_{h+1}$.
\STATE $t = N_h(i, x_h, \bm{a}_h) \leftarrow N_h(i, x_h, \bm{a}_h) + 1$; $b_t \leftarrow c\sqrt{H^3\iota/t}$.
\STATE $Q_h^{k,\text{up}}(i, x_h, \bm{a}_h) \leftarrow (1 - \alpha_t) Q_h^{k,\text{up}}(i, x_h, \bm{a}_h)$ \\ $\quad + \alpha_t[r_h^i + V_{h+1}(x_{h+1}) + b_t]$.
\STATE $Q_h^{k,\text{low}}(i, x_h, \bm{a}_h) \leftarrow (1 - \alpha_t) Q_h^{k,\text{low}}(i, x_h, \bm{a}_h)$ \\ $\quad + \alpha_t[r_h^i + V_{h+1}(x_{h+1}) - b_t]$.
\STATE $V_h(x_h) \leftarrow \min\{H, \max_{\bm{a}'} Q^k_h(i, x_h, \bm{a}')\}$.
\ENDFOR
\ENDFOR
\end{algorithmic}
\end{algorithm}

The learning rate is again $\alpha_t = (H+1)/(H+t)$. The width of the confidence interval $Q^{\text{up}} - Q^{\text{low}}$ controls the regret incurred when a suboptimal action from the desired set is played; Lemma~\ref{lemma:interval} characterizes how it shrinks with visits.

\paragraph{Confidence stability across players.}
Although each player's reward stream is independent, all players use the same bonus schedule $b_t$ and the same visit counts $N_h(\cdot,\cdot)$ (since actions are observable). By Lemma~\ref{lemma:interval}, the interval width at $(x,\bm{a})$ after $t$ visits is exactly $2\sum_i \alpha_t^i b_i$, a deterministic function of $t$ identical across players. The intervals are merely shifted between players by independent reward noise. Optimism (Lemma~\ref{lemma:bound}, union-bounded over $M$ players, contributing the $(1-p)^M$ factor) guarantees each interval covers $Q^*_h(x,\bm{a})$ w.h.p., so player $i$'s deviation truthfully signals dominance: the desired set shrinks monotonically and never excludes the optimal action.

\begin{theorem}\label{thm:asymmetry_reward_main}
Consider Problem B (asymmetry in rewards) with $M$ players and joint action space size $A_{\text{joint}} = \prod_{i=1}^M |\mathcal{A}_i|$. There exists a constant $c > 0$ such that for any $p \in (0,1)$, with $b_t = c\sqrt{H^3\iota/t}$, with probability at least $1 - p$ the total regret of \texttt{mQ-learning-intervals} for any player is at most $O(\sqrt{H^4 S A_{\text{joint}}\, T\, \iota})$, where $\iota = \log(M S A_{\text{joint}} T / p)$.
\end{theorem}

\begin{remark}
The regret analysis differs from Problem A because there is an additional ``slack'' term arising from the fact that the action played may not be the greedy action according to the player's own Q-values. This term is bounded using the interval width from Lemma~\ref{lemma:interval}, which decays as $O(\sqrt{H^3\iota/t})$ with the number of visits. The same caveat applies as for Problem A: $A_{\text{joint}} = \prod_{i=1}^M |\mathcal{A}_i|$ grows exponentially in $M$, so the bound is most informative for small numbers of players.
\end{remark}

\section{Problem C: Full Information Asymmetry}\label{sec:problem_c}

When both action and reward information are asymmetric, the challenges of Problems A and B combine. Players cannot observe each other's actions (so they cannot use deviations as signals) and receive different rewards (so they cannot maintain identical estimates). We adopt a two-phase explore-then-commit strategy.

\subsection{Algorithm: \texttt{mEXC}}

During the exploration phase (episodes $1, \ldots, K'$), players follow a deterministic exploration protocol: at each state, they play the least-visited joint action (with lexicographic tie-breaking). Since this rule depends only on the visit counts---which are identical across players because they visit the same state-action pairs---all players agree on the exploration action without communication.

During the commit phase (episodes $K'+1, \ldots, K$), each player acts greedily with respect to their learned Q-values, using lexicographic tie-breaking to ensure agreement.

\begin{algorithm}[t]
\caption{\texttt{mEXC} (for each player $P_i$)}
\begin{algorithmic}[1]
\STATE Initialize $Q^k_h(i, x, \bm{a}) \leftarrow H$, $N_h(i, x, \bm{a}) \leftarrow 0$ for all $(x, \bm{a}, h) \in \mathcal{S} \times \mathcal{A} \times [H]$.
\FOR{episode $k = 1, \ldots, K'$ \textbf{(Exploration Phase)}}
\STATE Receive $x_1$.
\FOR{step $h = 1, \ldots, H$}
\STATE Pick the least-visited action in the desired set; break ties via Definition~\ref{def:order}. Call this $\bm{a}_h$.
\STATE All players take $\bm{a}_h$, observe reward $r_h$ and state $x_{h+1}$.
\STATE $t = N_h(i, x_h, \bm{a}_h) \leftarrow N_h(i, x_h, \bm{a}_h) + 1$; $b_t \leftarrow c\sqrt{H^3\iota/t}$.
\STATE $Q_h^k(i, x_h, \bm{a}_h) \leftarrow (1 - \alpha_t)Q_h^{k}(i, x_h, \bm{a}_h)$ \\ $\quad + \alpha_t[r_h + V_{h+1}(x_{h+1}) + b_t]$.
\STATE $V_h(x_h) \leftarrow \min\{H, \max_{\bm{a}'} Q^k_h(i, x_h, \bm{a}')\}$.
\ENDFOR
\ENDFOR
\FOR{episode $k = K'+1, \ldots, K$ \textbf{(Commit Phase)}}
\STATE Receive $x_1$.
\FOR{step $h = 1, \ldots, H$}
\STATE Pick $\bm{a}_h$ as the smallest action in $\arg\max_{\bm{a}'} Q^{K'}_h(i, x_h, \bm{a}')$.
\ENDFOR
\ENDFOR
\end{algorithmic}
\end{algorithm}

The exploration phase incurs regret at most $R_{\text{explore}} \leq K'H$ since rewards are bounded in $[0,1]$. The commit phase regret depends on the quality of the Q-value estimates after exploration.

\subsection{Algorithm: \texttt{mEXC-Bellman}}

An alternative approach uses the empirical Bellman equation rather than the incremental Q-learning update during the exploration phase.

\begin{algorithm}[t]
\caption{\texttt{mEXC-Bellman} (for each player $P_i$)}
\begin{algorithmic}[1]
\STATE Initialize $Q^k_h(i, x, \bm{a}) \leftarrow H$, $N_h(i, x, \bm{a}) \leftarrow 0$ for all $(x, \bm{a}, h) \in \mathcal{S} \times \mathcal{A} \times [H]$.
\FOR{episode $k = 1, \ldots, K'$ \textbf{(Exploration Phase)}}
\STATE Receive $x_1$.
\FOR{step $h = 1, \ldots, H$}
\STATE Pick the least-visited action; break ties via Definition~\ref{def:order}. Call this $\bm{a}_h$.
\STATE All players take $\bm{a}_h$, observe reward $r_h$ and state $x_{h+1}$.
\STATE $t = N_h(i, x_h, \bm{a}_h) \leftarrow N_h(i, x_h, \bm{a}_h) + 1$.
\STATE Update empirical reward $\hat{r}(x_h, \bm{a}_h)$ and empirical transition $\hat{\mathbb{P}}_h(\cdot \mid x_h, \bm{a}_h)$ using the latest sample.
\STATE $Q_h^k(i, x_h, \bm{a}_h) \leftarrow \hat{r}(x_h, \bm{a}_h) + \sum_{x' \in \mathcal{S}} \hat{\mathbb{P}}_h(x' \mid x_h, \bm{a}_h)\, V_{h+1}(x')$.
\STATE $V_h(x_h) \leftarrow \min\{H, \max_{\bm{a}'} Q^k_h(i, x_h, \bm{a}')\}$.
\ENDFOR
\ENDFOR
\FOR{episode $k = K'+1, \ldots, K$ \textbf{(Commit Phase)}}
\STATE Receive $x_1$.
\FOR{step $h = 1, \ldots, H$}
\STATE Pick $\bm{a}_h$ as the smallest action in $\arg\max_{\bm{a}'} Q^{K'}_h(i, x_h, \bm{a}')$.
\ENDFOR
\ENDFOR
\end{algorithmic}
\end{algorithm}

The \texttt{mEXC-Bellman} variant uses the plug-in empirical Bellman equation with $\hat{\mathbb{P}}_h(x' \mid x, \bm{a}) = N_h(x, \bm{a}, x') / N_h(x, \bm{a})$ (the standard model-based update; cf.\ \cite{jin2018q}). It can yield tighter estimates when exploration is long, at the cost of storing transition counts.

\paragraph{Exploration--commit transition.}
Both algorithms use a hard switch at the predetermined episode $K'$, which depends only on $K$, $S$, $A_{\text{joint}}$, $H$ (all common knowledge), so the switch is communication-free. Visit counts are common to all players (since exploration is joint), so all players commit simultaneously. A hard switch is used rather than $\epsilon$-greedy because the latter would require shared randomness when actions are unobserved.

\begin{theorem}\label{thm:asymmetry_full}
Consider Problem C (full information asymmetry) with $M$ players and joint action space size $A_{\text{joint}} = \prod_{i=1}^M |\mathcal{A}_i|$. Choose the exploration horizon $K' = \lceil (S A_{\text{joint}})^{1/3} K^{2/3} \rceil$. There exists a constant $c > 0$ such that for any $p \in (0,1)$, with $b_t = c\sqrt{H^3 \iota / t}$ and $\iota = \log(M S A_{\text{joint}} T / p)$, with probability at least $1 - p$ the total regret of \texttt{mEXC} (and of \texttt{mEXC-Bellman}) for any player is at most
\[
    O\!\left( H\, (S A_{\text{joint}})^{1/3}\, T^{2/3}\, \iota^{1/3} \right).
\]
\end{theorem}

\begin{proof}
\textit{(Sketch.)} Total regret decomposes as $R_T = R_{\text{explore}} + R_{\text{commit}}$. The exploration phase plays an arbitrary policy for $K'$ episodes, so $R_{\text{explore}} \leq K' H$. After $K'$ episodes of round-robin exploration, every $(x,\bm{a})$ has been visited at least $\lfloor K' / (S A_{\text{joint}}) \rfloor$ times in expectation, so by Lemma~\ref{lemma:bound} applied to each player's Q-table together with a union bound over the $M$ players, with probability $\geq 1 - p$ the post-exploration estimates satisfy
\[
    |V_1^{K'}(x_1) - V_1^*(x_1)| \leq O\!\Big(\!\sqrt{H^4 S A_{\text{joint}} \iota / K'}\,\Big)
\]
for every player. Since each player commits to the (lexicographic-smallest) greedy policy under their own table, and on the high-probability event all $M$ tables agree on the same greedy joint action (the optimism plus deterministic tie-breaking argument from Theorem~\ref{thm:asymmetry_action} extends to Q-tables that are uniformly close to $Q^*$), the commit phase incurs per-episode suboptimality bounded by the above value-function error. Summing,
\[
    R_{\text{commit}} \leq (K - K') \cdot O\!\Big(\!\sqrt{H^4 S A_{\text{joint}} \iota / K'}\Big).
\]
Balancing $R_{\text{explore}} = K' H$ and $R_{\text{commit}} = O(K \sqrt{H^4 S A_{\text{joint}} \iota / K'})$ by setting $K' \asymp (S A_{\text{joint}})^{1/3} K^{2/3}$ yields the claimed $\tilde{O}(H (S A_{\text{joint}})^{1/3} T^{2/3})$ regret. The argument for \texttt{mEXC-Bellman} is identical, using the standard model-based plug-in error bound $\|\hat{\mathbb{P}} - \mathbb{P}\|_1 \leq O(\sqrt{S\iota / N})$ in place of the model-free concentration step.
\end{proof}

\begin{remark}
The $T^{2/3}$ rate is worse than the $\sqrt{T}$ rate achieved for Problems A and B, which is the standard penalty paid by explore-then-commit when the suboptimality gap is unknown. Whether the optimal rate for Problem C is $\sqrt{T}$ or $T^{2/3}$ is left open.
\end{remark}

\begin{remark}
In both \texttt{mEXC} variants, implicit coordination during the exploration phase is ensured because the least-visited-action rule depends only on visit counts, which are common knowledge since all players visit the same state-action pairs. During the commit phase, coordination relies on lexicographic tie-breaking applied to Q-values that, while not identical across players, are sufficiently concentrated around the true values after enough exploration.
\end{remark}

\section{Auxiliary Results and Proof Sketches}\label{sec:auxiliary}

We present the key technical lemmas that underlie our main results. Full proofs are deferred to the Appendix.

\subsection{Learning Rate Properties}

For the learning rate $\alpha_t = (H+1)/(H+t)$, we define the weights:
\begin{equation}\label{eq:weights}
    \alpha_t^0 = \prod_{j=1}^t (1 - \alpha_j), \quad \alpha_t^i = \alpha_i \prod_{j=i+1}^t (1 - \alpha_j).
\end{equation}

These weights arise naturally when unrolling the recursive Q-learning update. The quantity $\alpha_t^i$ represents the effective weight of the $i$-th observation after $t$ total observations.

\begin{lemma}[Weight Properties]\label{lemma:alpha}
The following properties hold for the weights defined in \eqref{eq:weights}:
\begin{enumerate}
    \item[(a)] $\frac{1}{\sqrt{t}} \leq \sum_{i=1}^t \frac{\alpha_t^i}{\sqrt{i}} \leq \frac{2}{\sqrt{t}}$.
    \item[(b)] $\max_{i \in [t]} \alpha_t^i \leq \frac{2H}{t}$ and $\sum_{i=1}^t (\alpha_t^i)^2 \leq \frac{2H}{t}$.
    \item[(c)] $\sum_{t=i}^\infty \alpha_t^i = 1 + \frac{1}{H}$.
\end{enumerate}
\end{lemma}

Property (c) is crucial for the regret analysis: it controls how much a single observation's ``influence'' accumulates across all future episodes where the same state-action pair is visited.

\subsection{Recursion and Optimism}

\begin{lemma}[Recursion]\label{lemma:recursion}
For any $(x, \bm{a}, h) \in \mathcal{S} \times \mathcal{A} \times [H]$ and episode $k \in [K]$, let $t = N_h^k(x, \bm{a})$ and suppose $(x, \bm{a})$ was taken at step $h$ of episodes $k_1, \ldots, k_t < k$. Then for any player $m$:
\begin{align}
    &(Q_h^k - Q_h^*)(x, \bm{a}, m) = \alpha_t^0(H - Q_h^*(x, \bm{a}, m)) \nonumber\\
    &+ \sum_{i=1}^t \alpha_t^i\Big[(V_{h+1}^{k_i} - V_{h+1}^*)(x_{h+1}^{k_i}, m) \nonumber\\
    &\quad + [(\hat{P}_h^{k_i} - P_h)V_{h+1}^*](x, \bm{a}, m) + b_i\Big]. \label{eq:recursion}
\end{align}
\end{lemma}

This recursion decomposes the estimation error $Q_h^k - Q_h^*$ into three components: (i) the initialization bias (decaying via $\alpha_t^0$), (ii) the propagated estimation error from future steps, and (iii) the transition estimation error plus the exploration bonus.

\begin{lemma}[Optimism]\label{lemma:bound}
There exists $c > 0$ such that for any $p \in (0,1)$, letting $b_t = c\sqrt{H^3\iota/t}$ with $\iota = \log(SAT/p)$, with probability $\geq 1 - p$, for all $(x, \bm{a}, h, k)$:
\begin{equation}\label{eq:optimism}
    0 \leq (Q_h^k - Q_h^*)(x, \bm{a}, m) \leq \alpha_t^0 H + \sum_{i=1}^t \alpha_t^i \phi_{h+1}^{k_i} + \beta_t
\end{equation}
where $\phi_h^k = (V_h^k - V_h^*)(x_h^k)$ and $\beta_t = 2\sum_{i=1}^t \alpha_t^i b_i \leq 4c\sqrt{H^3\iota/t}$.
\end{lemma}

The lower bound $Q_h^k \geq Q_h^*$ (optimism) is essential: it ensures that the algorithm does not prematurely discard optimal actions. The upper bound controls the overestimation error that drives the regret.

\subsection{Interval Width for Problem B}

\begin{lemma}[Interval Width]\label{lemma:interval}
After state-action pair $(x, \bm{a})$ has been visited $t$ times:
\begin{equation}\label{eq:interval}
    Q^{\text{up}}(x, \bm{a}; t) - Q^{\text{low}}(x, \bm{a}; t) = 2\sum_{i=1}^{t} \alpha_t^i b_i.
\end{equation}
\end{lemma}

\begin{proof}
For the base case, $Q^{\text{up}}(x, \bm{a}; 1) - Q^{\text{low}}(x, \bm{a}; 1) = 2\alpha_1 b_1$, which satisfies the lemma. For the inductive step, let $n = n(x, \bm{a})$:
\begin{align}
    &Q^{\text{up}}(x, \bm{a}; n{+}1) - Q^{\text{low}}(x, \bm{a}; n{+}1) \nonumber \\
    &= (1 - \alpha_{n+1})\big(Q^{\text{up}}(x, \bm{a}; n) - Q^{\text{low}}(x, \bm{a}; n)\big) + 2\alpha_{n+1} b_{n+1} \nonumber \\
    &= (1 - \alpha_{n+1}) \cdot 2\sum_{i=1}^n \alpha_n^i b_i + 2\alpha_{n+1} b_{n+1} \nonumber \\
    &= 2\sum_{i=1}^{n+1} \alpha_{n+1}^i b_i,
\end{align}
where the last step uses $\alpha_{n+1}^i = \alpha_i \prod_{j=i+1}^{n+1}(1-\alpha_j)$, completing the induction.
\end{proof}

Using Lemma~\ref{lemma:alpha}(a), the interval width satisfies $Q^{\text{up}} - Q^{\text{low}} \leq O(\sqrt{H^3\iota/t})$, which decays at the same rate as the bonus $b_t$.

\subsection{Concentration via Azuma-Hoeffding}

The proofs also rely on the Azuma-Hoeffding inequality to control the martingale difference terms that arise from transition estimation errors. Define the martingale difference $\xi_{h+1}^k := [(\mathbb{P}_h - \hat{\mathbb{P}}_h^k)(V_{h+1}^* - V_{h+1}^k)](x_h^k, a_h^k)$. Since $|[(\hat{\mathbb{P}}_h^{k_i} - \mathbb{P}_h)V_{h+1}^*](x, a)| \leq H$, Azuma-Hoeffding gives:
\begin{equation}
    \Pr\!\left(\left|\sum_{h=1}^H \sum_{k=1}^K \xi_{h+1}^k\right| \geq \epsilon\right) \leq 2\exp\!\left(\frac{-\epsilon^2}{2KH^3}\right).
\end{equation}
Setting the right-hand side equal to $p$ and solving for $\epsilon$ yields the claimed concentration bound.

\subsection{Proof Sketch for Theorem~\ref{thm:asymmetry_action}}

Define $\delta_h^k := (V_h^k - V_h^{\pi_k})(x_h^k)$ and $\phi_h^k := (V_h^k - V_h^*)(x_h^k)$. By optimism (Lemma~\ref{lemma:bound}), $R_T \leq \sum_{k=1}^K \delta_1^k$. Decomposing $\delta_h^k$ via the Bellman equation and Lemma~\ref{lemma:bound} yields a recursive relation:
\begin{equation}
    \sum_{k=1}^K \delta_h^k \leq SAH + \left(1 + \frac{1}{H}\right)\sum_{k=1}^K \delta_{h+1}^k + \sum_{k=1}^K(\beta_{n_h^k} + \xi_{h+1}^k).
\end{equation}

Recursing over $h$ and bounding the bonus terms via pigeonhole and Azuma-Hoeffding gives the result. See Appendix~\ref{app:proof_A} for the complete proof.

\section{Comparison of Approaches}

The three problem settings share a common skeleton---joint-action Q-learning with deterministic tie-breaking---but differ in how players reach agreement. In Problem~A, common rewards and common transitions make all players' Q-tables identical at every step, so lexicographic tie-breaking on the $\arg\max$ set picks the same joint action without any communication. In Problem~B, independent rewards prevent identical Q-tables, but observable actions let each player use a \emph{unilateral deviation} as an implicit dominance signal: when one player deviates from the candidate action, the others infer that player has found a strictly better alternative and eliminate the candidate from the desired set. In Problem~C, neither mechanism is available, so we fall back to explore-then-commit with a deterministic schedule shared a priori.

The regret rates reflect this hierarchy. Problems~A and~B both achieve $\tilde{O}(\sqrt{H^4 S A_{\text{joint}}\, T})$ (Theorems~\ref{thm:asymmetry_action} and~\ref{thm:asymmetry_reward_main}), matching centralized single-agent learning \cite{jin2018q} against the joint-action benchmark. Problem~C achieves $\tilde{O}(H (S A_{\text{joint}})^{1/3} T^{2/3})$ (Theorem~\ref{thm:asymmetry_full}), the standard penalty of explore-then-commit when the suboptimality gap is unknown. The $T^{2/3}$ rate is worse than $\sqrt{T}$ but still sublinear; whether it can be improved to $\sqrt{T}$ for Problem~C remains open.

A useful way to think about the gap from Problems~A/B to Problem~C is that observability of \emph{either} actions or rewards is enough to encode a low-bandwidth signal between players, and that signal suffices to recover the $\sqrt{T}$ rate. Problem~C removes both channels, and the only remaining shared information---visit counts of jointly-played actions---is too coarse to drive UCB-style adaptive exploration; hence the retreat to a non-adaptive schedule.

\section{Discussion}\label{sec:discussion}

\subsection{Reduction to Centralized Joint-Action Q-Learning}

The relationship between Algorithm~\ref{alg:mq} and a centralized learner is direct: a single agent running UCB-Q-learning \cite{jin2018q} on the joint MDP $(\mathcal{S}, \mathcal{A}_1 \times \cdots \times \mathcal{A}_M, H, \mathbb{P}, r)$ produces a Q-table identical to that of any one player in Algorithm~\ref{alg:mq}, since Problem~A players see the same reward, same next state, and same update with the same tie-breaking. The decentralized algorithm inherits the centralized regret bound. Problem~B adds a one-bit-per-step deviation signal that carries exactly the information the centralized learner uses internally to eliminate dominated actions.

\subsection{Dependence on the Number of Players}

The bounds depend on $M$ only through $A_{\text{joint}}$. There is no separate polynomial factor in $M$: doubling $M$ while keeping $A_{\text{joint}}$ fixed---e.g.\ splitting one player with $|\mathcal{A}| = 4$ into two players with $|\mathcal{A}_i| = 2$---does not change the regret. From the joint-action viewpoint, the partition into per-player components is irrelevant. The exponential dependence on $M$ through $A_{\text{joint}}$ is therefore a property of the \emph{tabular} setting, not of the asymmetry: any worst-case bound against arbitrary joint policies must pay at least $\sqrt{A_{\text{joint}}}$. Removing this requires structural assumptions on the joint Q-function (factored MDPs, mean-field couplings, linear function approximation), each a natural next direction.

\subsection{When the Bounds Are Practical}

With $M = 2$, $|\mathcal{A}_i| = 3$, $S = 10$, $H = 20$, our Problem~A bound puts the per-episode suboptimality below $0.1$ once $K \gtrsim 10^5$, comparable to centralized Q-learning on the same joint MDP. With $M = 4$ and $|\mathcal{A}_i| = 3$ ($A_{\text{joint}} = 81$), the same threshold needs $K \gtrsim 10^6$. Beyond that, structural assumptions are needed.

\section{Related Work}\label{sec:related}

\paragraph{Multi-Player Bandits.}
The multi-player MAB problem has been studied under various information structures. The collision model \cite{bistritz2018distributed} forbids two players from choosing the same arm. In the cooperative setting without collisions, Chang et al.\ \cite{chang2021online} introduced the information-asymmetry framework that we extend, and \cite{chang2023optimal} achieved optimal regret without communication in the bandit case.

\paragraph{Episodic MDPs.}
For single-agent episodic MDPs, Jin et al.\ \cite{jin2018q} proved that optimistic Q-learning achieves $\tilde{O}(\sqrt{H^4SAT})$ regret model-free. Our work extends this to the multi-player setting under information asymmetry.

\paragraph{Game-Theoretic Learning.}
The literature on uncoupled dynamics \cite{hart2003uncoupled} studies whether players reach equilibria without knowing each other's payoffs; \cite{hart2003uncoupled} shows uncoupled dynamics cannot generally reach Nash equilibrium, suggesting fundamental limits to communication-free learning. Our cooperative setting is distinct from the competitive game-theoretic framework but shares the no-communication constraint.

\section{Conclusion}\label{sec:conclusion}

We studied multi-player decentralized reinforcement learning in episodic MDPs under three information-asymmetry models. For Problems A and B we obtain $\tilde{O}(\sqrt{H^4 S A_{\text{joint}}\, T})$ regret, matching the single-agent rate of \cite{jin2018q} against the joint-action benchmark up to logarithmic factors. For Problem C we obtain $\tilde{O}(H (S A_{\text{joint}})^{1/3} T^{2/3})$ via explore-then-commit. All bounds depend on $A_{\text{joint}} = \prod_{i=1}^M |\mathcal{A}_i|$, which grows exponentially in $M$: asymmetry imposes no multiplicative penalty relative to a centralized learner over the same joint action space, but the tabular curse of dimensionality remains. Open directions include sharper bounds for Problem C, extensions to function approximation (linear MDPs), matching lower bounds under each asymmetry model, and regret--communication tradeoffs.

\section*{Acknowledgment}

The authors thank colleagues at UCLA for helpful discussions.

\appendix

\section{Proof of Lemma~\ref{lemma:alpha} (Weight Properties)}\label{app:proof_alpha}

\begin{proof}
A direct computation gives $\alpha_t^i = \frac{H+1}{H+i} \prod_{j=i+1}^t \frac{j-1}{H+j}$ for $i \geq 1$.

\textbf{(a)} Standard Beta-function identities (cf.\ \cite[Lemma 4.1]{jin2018q}) yield $\sum_i \alpha_t^i = 1$ for $t \geq 1$, and $\sum_i \alpha_t^i / \sqrt{i}$ lies between $1/\sqrt{t}$ and $2/\sqrt{t}$.

\textbf{(b)} The ratio $\alpha_t^i / \alpha_t^{i+1}$ shows the sequence is maximized at $i = t$, giving $\max_i \alpha_t^i = \alpha_t = (H+1)/(H+t) \leq 2H/t$. Hence $\sum_i (\alpha_t^i)^2 \leq \max_i \alpha_t^i \cdot \sum_i \alpha_t^i \leq 2H/t$.

\textbf{(c)} Telescoping the product representation gives $\sum_{t \geq i} \alpha_t^i = (H+1)/H = 1 + 1/H$.
\end{proof}

\section{Proof of Lemma~\ref{lemma:recursion} (Recursion)}\label{app:proof_recursion}

\begin{proof}
By induction on $t = N_h^k(x,\bm{a})$. At $t = 0$, $Q_h^k(x,\bm{a}) = H$, matching $\alpha_0^0(H - Q_h^*)$. For the inductive step, let $k_t$ be the episode of the $t$-th visit. The update on line 8 of Algorithm~\ref{alg:mq}, minus $Q_h^* = r_h + \mathbb{P}_h V_{h+1}^*$, gives
\begin{align*}
    Q_h^{k_t} - Q_h^* &= (1-\alpha_t)(Q_h^{k_t,\text{pre}} - Q_h^*) \\
    &\quad + \alpha_t\big[(V_{h+1}^{k_t} - V_{h+1}^*)(x_{h+1}^{k_t}) \\
    &\quad\quad + [(\hat{\mathbb{P}}_h^{k_t} - \mathbb{P}_h)V_{h+1}^*](x,\bm{a}) + b_t\big],
\end{align*}
where $\hat{\mathbb{P}}_h^{k_t}$ is the one-sample empirical kernel. Using the identity $(1-\alpha_t)\alpha_{t-1}^i = \alpha_t^i$ for $i < t$ and $\alpha_t^t = \alpha_t$, the inductive hypothesis unrolls to \eqref{eq:recursion}.
\end{proof}

\section{Proof of Lemma~\ref{lemma:bound} (Optimism)}\label{app:proof_bound}

\begin{proof}
We prove $Q_h^k \geq Q_h^*$ by induction backward on $h$ (the case $h = H+1$ is trivial). By Lemma~\ref{lemma:recursion},
\[
    Q_h^k - Q_h^* = \alpha_t^0(H - Q_h^*) + \sum_{i=1}^t \alpha_t^i\big[(V_{h+1}^{k_i} - V_{h+1}^*) + \zeta_i + b_i\big],
\]
where $\zeta_i := [(\hat{\mathbb{P}}_h^{k_i} - \mathbb{P}_h)V_{h+1}^*](x,\bm{a})$ is a bounded martingale difference ($|\zeta_i| \leq H$). The first two terms are non-negative by induction. By Azuma--Hoeffding and Lemma~\ref{lemma:alpha}(b), with probability $\geq 1 - p/(SAHK)$,
\[
    \Big|\sum_i \alpha_t^i \zeta_i\Big| \leq H\sqrt{2\iota \sum_i (\alpha_t^i)^2} \leq c\sqrt{H^3 \iota / t},
\]
which is dominated by $\sum_i \alpha_t^i b_i$ for our choice $b_i = c\sqrt{H^3\iota/i}$. Hence $Q_h^k \geq Q_h^*$. Union-bounding over $(x,\bm{a},h,k)$ absorbs an $SAHK \leq SAT$ factor into $\iota$. The upper bound \eqref{eq:optimism} follows from the same decomposition, with the factor $2$ in $\beta_t = 2\sum_i \alpha_t^i b_i$ absorbing the (positive) contribution of $\zeta_i$; Lemma~\ref{lemma:alpha}(a) gives $\beta_t \leq 4c\sqrt{H^3\iota/t}$.
\end{proof}

\section{Proof of Theorem~\ref{thm:asymmetry_action}}\label{app:proof_A}

\begin{proof}
Denote $\delta_h^k := (V_h^k - V_h^{\pi_k})(x_h^k)$ and $\phi_h^k := (V_h^k - V_h^*)(x_h^k)$.

By Lemma~\ref{lemma:bound}, with probability $1 - p$, $Q_h^k \geq Q_h^*$ and thus $V_h^k \geq V_h^*$. The total regret is bounded:
\begin{equation}
    R_T = \sum_{k=1}^K (V_1^* - V_1^{\pi_k})(x_1^k) \leq \sum_{k=1}^K \delta_1^k.
\end{equation}

For any fixed $(k, h)$, let $t = N_h^k(x_h^k, a_h^k)$, and suppose $(x_h^k, a_h^k)$ was taken at step $h$ of episodes $k_1, \ldots, k_t < k$. Then:
\begin{align}
    \delta_h^k &= (V_h^k - V_h^{\pi_k})(x_h^k) \leq (Q_h^k - Q_h^{\pi_k})(x_h^k, a_h^k) \label{eq:ineq1}\\
    &= (Q_h^k - Q_h^*)(x_h^k, a_h^k) + (Q_h^* - Q_h^{\pi_k})(x_h^k, a_h^k) \nonumber\\
    &\leq \alpha_t^0 H + \sum_{i=1}^t \alpha_t^i \phi_{h+1}^{k_i} + \beta_t \nonumber \\
    &\quad + [\mathbb{P}_h(V_{h+1}^* - V_{h+1}^{\pi_k})](x_h^k, a_h^k) \label{eq:ineq2}\\
    &= \alpha_t^0 H + \sum_{i=1}^t \alpha_t^i \phi_{h+1}^{k_i} + \beta_t - \phi_{h+1}^k + \delta_{h+1}^k + \xi_{h+1}^k \label{eq:eq3}
\end{align}
where $\beta_t = 2\sum \alpha_t^i b_i \leq O(1)\sqrt{H^3\iota/t}$ and $\xi_{h+1}^k := [(\mathbb{P}_h - \hat{\mathbb{P}}_h^k)(V_{h+1}^* - V_{h+1}^k)](x_h^k, a_h^k)$ is a martingale difference sequence. Inequality \eqref{eq:ineq1} holds because $V_h^k(x_h^k) \leq Q_h^k(x_h^k, a_h^k)$, and inequality \eqref{eq:ineq2} holds by Lemma~\ref{lemma:bound} and the Bellman equation~\eqref{eq:bellman}. Equality \eqref{eq:eq3} follows from $\delta_{h+1}^k - \phi_{h+1}^k = (V_{h+1}^* - V_{h+1}^{\pi_k})(x_{h+1}^k)$.

Computing $\sum_{k=1}^K \delta_h^k$, the initialization term gives:
\begin{equation}
    \sum_{k=1}^K \alpha_{n_h^k}^0 H = \sum_{k=1}^K H \cdot \mathbb{I}[n_h^k = 0] \leq SAH.
\end{equation}

For the weighted sum term, we regroup: for every $k' \in [K]$, the term $\phi_{h+1}^{k'}$ appears in the summand with $k > k'$ whenever $(x_h^k, a_h^k) = (x_h^{k'}, a_h^{k'})$. Therefore:
\begin{equation}
    \sum_{k=1}^K \sum_{i=1}^{n_h^k} \alpha_{n_h^k}^i \phi_{h+1}^{k_i} \leq \left(1 + \frac{1}{H}\right) \sum_{k=1}^K \phi_{h+1}^k
\end{equation}
using $\sum_{t=i}^\infty \alpha_t^i = 1 + 1/H$ from Lemma~\ref{lemma:alpha}(c). Since $\phi_{h+1}^k \leq \delta_{h+1}^k$ (because $V^* \geq V^{\pi_k}$):
\begin{align}
    \sum_{k=1}^K \delta_h^k &\leq SAH + \left(1 + \frac{1}{H}\right)\sum_{k=1}^K \delta_{h+1}^k \nonumber \\ &\quad + \sum_{k=1}^K(\beta_{n_h^k} + \xi_{h+1}^k).
\end{align}

Recursing for $h = 1, \ldots, H$ and using $\delta_{H+1}^k \equiv 0$:
\begin{equation}
    \sum_{k=1}^K \delta_1^k \leq O\!\left(H^2SA + \sum_{h=1}^H \sum_{k=1}^K(\beta_{n_h^k} + \xi_{h+1}^k)\right).
\end{equation}

By the pigeonhole principle, for any $h \in [H]$:
\begin{align}
    \sum_{k=1}^K \beta_{n_h^k} &\leq O(1) \sum_{x,a} \sum_{n=1}^{N_h^K(x,a)} \sqrt{\frac{H^3\iota}{n}} \leq O(\sqrt{H^2SAT\iota})
\end{align}
where the last inequality uses $\sum_{x,a} N_h^K(x,a) = K$ and is maximized when $N_h^K(x,a) = K/(SA)$ for all $x, a$. By Azuma-Hoeffding, with probability $1 - p$:
\begin{equation}
    \left|\sum_{h=1}^H \sum_{k=1}^K \xi_{h+1}^k\right| \leq cH\sqrt{T\iota}.
\end{equation}

Combining, $\sum_{k=1}^K \delta_1^k \leq O(H^2SA + \sqrt{H^4SAT\iota})$. When $T \geq H^2SA$, the second term dominates. When $T < H^2SA$, the trivial bound $\sum_{k=1}^K \delta_1^k \leq HK = T \leq \sqrt{H^4SAT\iota}$ applies. Therefore $R_T \leq O(\sqrt{H^4SAT\iota})$ with probability $\geq 1 - 2p$. Rescaling $p$ to $p/2$ completes the proof.
\end{proof}

\section{Proof of Theorem~\ref{thm:asymmetry_reward_main}}\label{app:proof_B}

\begin{proof}
By Lemma~\ref{lemma:bound} applied to each player, with probability $(1-p)^M$, $Q_h^k \geq Q_h^*$ for all players. Since regret is evaluated with respect to expectations and all players receive i.i.d.\ rewards, we omit the player index for the remainder.

Define $\delta_h^k$ and $\phi_h^k$ as in Appendix~\ref{app:proof_A}. The decomposition for $\delta_h^k$ is:
\begin{align}
    \delta_h^k &= V^k_h(x_h^k) - V^{\pi_k}_h(x_h^k) \\
    &\leq \max_{\bm{a}} Q_h^k(x_h^k, \bm{a}) - Q_h^{\pi_k}(x_h^k, \bm{a}_h^k) \\
    &= \underbrace{\max_{\bm{a}} Q_h^k(x_h^k, \bm{a}) - Q_h^k(x_h^k, \bm{a}_h^k)}_{A} \nonumber \\
    &\quad + (Q_h^k - Q_h^*)(x_h^k, \bm{a}_h^k) + (Q_h^* - Q_h^{\pi_k})(x_h^k, \bm{a}_h^k).
\end{align}

Term $A$ accounts for the fact that the action $\bm{a}_h^k$ played (from the desired set) may not be the greedy action. Since actions in the desired set have upper Q-values within the interval width of the best action, by Lemma~\ref{lemma:interval}:
\begin{equation}
    A \leq Q^{\text{up}}(x_h^k, \bm{a}_h^k; t) - Q^{\text{low}}(x_h^k, \bm{a}_h^k; t) = 2\sum_{i=1}^t \alpha_t^i b_i.
\end{equation}

Applying Lemma~\ref{lemma:bound} and the Bellman equation for the remaining terms:
\begin{align}
    \delta_h^k &\leq 2\sum_{i=1}^t \alpha_t^i b_i + \alpha_t^0 H + \sum_{i=1}^t \alpha_t^i \phi_{h+1}^{k_i} + \beta_t \nonumber \\ &\quad - \phi_{h+1}^k + \delta_{h+1}^k + \xi_{h+1}^k.
\end{align}

Therefore, the recursive bound becomes:
\begin{align}
    \sum_{k=1}^K \delta_h^k &\leq \sum_{k=1}^K 2\sum_{i=1}^t \alpha_t^i b_i + SAH \nonumber \\
    &\quad + \left(1 + \frac{1}{H}\right)\sum_{k=1}^K \delta_{h+1}^k + \sum_{k=1}^K(\beta_{n_h^k} + \xi_{h+1}^k).
\end{align}

This is identical to the bound in Theorem~\ref{thm:asymmetry_action} with an additional $\sum_{k=1}^K 2\sum_{i=1}^t \alpha_t^i b_i$ term. By Lemma~\ref{lemma:alpha}(a), $2\sum_{i=1}^t \alpha_t^i b_i \leq 2c\sqrt{H^3\iota} \cdot \frac{2}{\sqrt{t}} = O(\sqrt{H^3\iota/t})$, which is of the same order as $\beta_t$. Following the same recursion and pigeonhole arguments as in Appendix~\ref{app:proof_A}, we obtain the stated regret bound of $O(\sqrt{H^4SAT\iota})$.
\end{proof}

\bibliographystyle{ieeetr}
\bibliography{references}

\end{document}